\documentclass{article} 
\usepackage[preprint]{nips_2018}

\usepackage{dsfont}
\usepackage{amsfonts}
\usepackage{amsmath}
\usepackage{slashbox}

\usepackage{amsthm}

\usepackage{caption}
\usepackage{subcaption}
\usepackage{tabularx,booktabs}
\usepackage{slashbox}
\usepackage{booktabs,caption}
\usepackage[flushleft]{threeparttable}
\usepackage{multirow}
\usepackage{mathtools}
\usepackage{color}

\usepackage{amssymb}
\usepackage{soul}
\usepackage{cancel}
\usepackage{makecell}
\usepackage{multirow}
\usepackage{xcolor}

\usepackage[utf8]{inputenc} 
\usepackage[T1]{fontenc}    
\usepackage{hyperref}       
\usepackage{url}            
\usepackage{booktabs}       
\usepackage{amsfonts}       
\usepackage{nicefrac}       
\usepackage{microtype}      

\title{Multi-variable LSTM neural network for autoregressive exogenous model}

\author{
  Tian Guo$^*$  \\
  ETH Zurich, Switzerland\\
  \texttt{\{tian.guo\}@gess.ethz.ch} \\
  \And
  Tao Lin \thanks{Equal contribution.} \\
  EPFL Lausanne, Switzerland \\
  \texttt{\{tao.lin\}@epfl.ch}
}

\newcommand{\cb}[1]{\colorbox{blue!10}{#1}}

\newtheorem{theorem}{Theorem}

\begin{document}

\maketitle

\begin{abstract}
In this paper, we propose multi-variable LSTM capable of accurate forecasting and variable importance interpretation for time series with exogenous variables.
Current attention mechanism in recurrent neural networks mostly focuses on the temporal aspect of data and falls short of characterizing variable importance.
To this end, the multi-variable LSTM equipped with tensorized hidden states is developed to learn hidden states for individual variables, which give rise to our mixture temporal and variable attention. Based on such attention mechanism, we infer and quantify variable importance.    
Extensive experiments using real datasets with Granger-causality test and the synthetic dataset with ground truth demonstrate the prediction performance and interpretability of multi-variable LSTM in comparison to a variety of baselines. 
It exhibits the prospect of multi-variable LSTM as an end-to-end framework for both forecasting and knowledge discovery. 
\end{abstract}

\section{Introduction}

Our daily life is now surrounded by various types of sensors, ranging from smartphones, video cameras, Internet of things, to robots. The observations yield by such devices over time are naturally organized in time series data \cite{qin2017dual, yang2015deep}.
In this paper, we focus on time series with exogenous variables. Specifically, given a target time series as well as an additional set of time series corresponding to exogenous variables, a predictive model using the historical observations of both target and exogenous variables to predict the future values of the target variable is an autoregressive exogenous model, referred to as ARX.
ARX models have been successfully used for modeling the input-output behavior of many complex systems \cite{ dipietro2017analyzing, zemouri2010defining, lin1996learning}.
In addition to forecasting, the interpretability of such models is essential for deployment,
e.g. understanding the different importance of exogenous variables w.r.t. the evolution of the target one \cite{hu2018listening, siggiridou2016granger, zhou2015probabilistic}.



Meanwhile, long short-term memory units (LSTM)~\cite{hochreiter1997long} and the gated recurrent unit (GRU)~\cite{cho2014properties}, a class of recurrent neural networks (RNN), have achieved great success in various applications on sequence and time series data \cite{lipton2015learning, wang2016morphological, guo2016robust, lin2017hybrid, sutskever2014sequence}.

However, current recurrent neural networks fall short of achieving interpretability on the variable level when they are used for ARX models. 
For instance, when fed with the multi-variable historical observations of the target and exogenous variables, LSTM blindly blends the information of all variables into the memory cells and hidden states which are used for prediction. Therefore, it is intractable to distinguish the contribution of individual variables into the prediction by looking into hidden states \cite{zhang2017stock}. 

Recently, attention-based neural networks~\cite{bahdanau2014neural, vinyals2015pointer, chorowski2015attention, choi2016retain, qin2017dual, cinar2017position} have been proposed to enhance the ability of RNN in selectively using long-term memory as well as the interpretability. Current attention mechanism is mostly applied to hidden states across time steps, thereby focusing on capturing temporally important information and failing to uncover the different importance of input variables. 

To this end, 
we aim to develop a LSTM neural network based ARX model to achieve a unified framework of both forecasting and knowledge discovery. In particular, the contribution is fourfold.
First, we propose the multi-variable LSTM, referred to as MV-LSTM, with tensorized hidden states and associated updating scheme, such that each element of the hidden state tensor encodes information for a certain input variable. 
Second, by using the variable-wise hidden states we develop a probabilistic mixture representation of temporal and variable attention. Learning and forecasting of MV-LSTM are built on top of this mixture attention mechanism. 
Third, we propose to interpret and quantify variable importance by the posterior inference of variable attention. 
Lastly, we perform an extensive experimental evaluation of MV-LSTM against statistical, machine learning and neural network baselines to demonstrate the prediction performance and interpretability of MV-LSTM.
The idea of MV-LSTM easily applies to other variants of RNN, e.g., GRU or stacking multiple 
MV-LSTM layers. These will be the future work. 
\section{Related work}
Vanilla recurrent neural networks have been used to study nonlinear ARX problem in 
\citet{zemouri2010defining, diaconescu2008use, dipietro2017analyzing}.
\cite{tank2017interpretable, tank2018neural} proposed to identify causal variables w.r.t. the target one via sparse regularization. 
Our MV-LSTM is intended for providing the accurate prediction as well as interpretability of variable importance via attention mechanism.


Recently, attention mechanism has gained increasing popularity
due to its ability in enabling recurrent neural networks to select parts of hidden states across time steps as well as enhancing the interpretability of networks \cite{bahdanau2014neural, vinyals2015pointer, choi2016retain, vaswani2017attention, lai2017modeling, qin2017dual, cinar2017position, choi2018fine, guo2018interpretable}. 
However, current attention mechanism is normally applied to hidden states across time steps, and for multi-variable input sequence, it fails to characterize variable level importance. Only some very recent studies~\cite{choi2016retain,qin2017dual} attempted to develop attention mechanism capable of handling multi-variable sequence data.
\cite{qin2017dual, choi2016retain} first use neural networks to learn weights on input variables and then feed weighted input data into another neural network \cite{qin2017dual} or use it directly for forecasting \cite{choi2016retain}.
In our MV-LSTM, temporal and variable attention are jointly derived from hidden states for individual variables learned via one end-to-end network.

Another line of related research is about tensorization and selectively updating of hidden states in recurrent neural networks. 
\cite{novikov2015tensorizing, do2017matrix} proposed to represent hidden states as a matrix. 
\cite{he2017wider} developed tensorized LSTM in which hidden states are represented by tensors to enhance the capacity of networks without additional parameters.
\cite{koutnik2014clockwork, neil2016phased, kuchaiev2017factorization} put forward to partition the hidden layer into separated modules as independent feature groups. In MV-LSTM, hidden states are organized in a matrix, each element of which encodes information specific to an input variable. Meanwhile, the hidden states are correlatively updated such that inter-correlation among input variables is still captured.

\section{Multi-Variable LSTM}\label{sec:mv}
Assume we have $N-1$ exogenous time series and a target series $\mathbf{y}$ of length $T$, where $\mathbf{y} = [ y_1, \cdots, y_T ]$ and $\mathbf{y} \in \mathbb{R}^T$.\footnote{Vectors are assumed to be in column form throughout this paper.}
By stacking exogenous time series and target series, we define a multi-variable input sequence as
$\mathbf{X}_{T} = \{ \mathbf{x}_1, \cdots, \mathbf{x}_T  \}$, 
where $ \mathbf{x}_t = [ x_{t}^1, \cdots, x_{t}^{N-1}, y_{t} ]  \in \mathbb{R}^N$ is the multi-variable input at time step $t$ and $x_{t}^n \in \mathbb{R}$ is the observation of $n$-th exogenous time series at time $t$.
Given $\mathbf{X}_T$, we aim to learn a non-linear mapping to predict the next value of the target series, namely $\hat{y}_{T+1} = \mathcal{F}(\mathbf{X}_T)$.
Model $\mathcal{F}(\cdot)$ should be interpretable in the sense that we can understand which exogenous variables are crucial for the prediction.      
\subsection{Network Architecture}
Inspired by \cite{he2017wider, kuchaiev2017factorization}, in MV-LSTM we develop tensorized hidden states and associated update scheme, which are able to ensure that each element of the hidden state tensor encapsulates information exclusively from a certain variable of the input. As a result, it enables to develop a flexible temporal and variable attention mechanism on top of such hidden states. 
\begin{figure}[htbp!]
\centering
\includegraphics[width=0.8\textwidth]{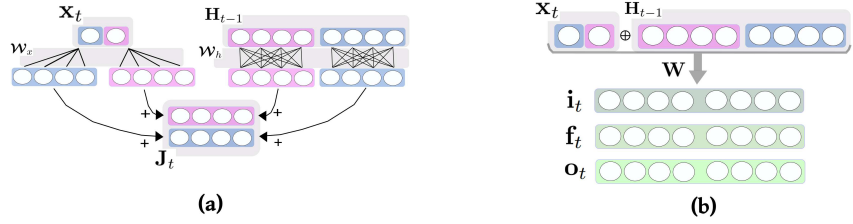}
\caption{A toy example of a MV-LSTM with a two-variable input sequence and the hidden matrix of size $4 \times 2$, i.e. $4$-dimensional hidden state per variable.
Panel (a) exhibits the derivation of cell update matrix $\mathbf{J}_t$, i.e. Eq. \eqref{eq:j}. Purple and blue colors correspond to two variables. Rectangles with circles inside represent input data and hidden states at one step. Grey areas outline the corresponding weights, hidden states, and input. Panel (b) demonstrates the process of gate calculation, i.e. Eq. \eqref{eq:gate}. (best viewed in color)}
\label{fig:graph}
\end{figure}

Specifically, we define the hidden state tensor (matrix) at time step $t$ in a MV-LSTM layer as $\mathbf{H}_t = [ \mathbf{h}_t^{1} \, , \cdots, \, \mathbf{h}_t^{N} ]^\top$, 
where $\mathbf{H}_t \in \mathbb{R}^{N \times d}$, $\mathbf{h}_t^{n} \in \mathbb{R}^{d}$, $N \cdot d = D$ and $D$ is overall size of the layer. 
The element $\mathbf{h}_t^{n}$ of $\mathbf{H}_t$ is a hidden state vector specific to $n$-th input variable.
Then, we define the input-to-hidden transition tensor (matrix) as 
$\boldsymbol{\mathcal{W}}_x = [ \mathbf{W}_x^{1} \, , \cdots, \, \mathbf{W}_x^{N} ]^\top$, 
where $\boldsymbol{\mathcal{W}}_x \in \mathbb{R}^{N \times d }$ and $\mathbf{W}_x^{n} \in \mathbb{R}^{d} $. 
The hidden-to-hidden transition tensor is defined as: 
$\boldsymbol{\mathcal{W}}_h = [ \mathbf{W}_h^{1} \, , \cdots, \, \mathbf{W}_h^{N} ]$, 
where $\boldsymbol{\mathcal{W}}_h \in \mathbb{R}^{N \times d \times d}$ and 
$\mathbf{W}_h^{n} \in \mathbb{R}^{d \times d}$.

Similar to the standard LSTM neural networks \cite{hochreiter1997long}, 
MV-LSTM has the input, forget and output gates as well as the memory cells to control the update of hidden state matrix.
Given the newly incoming input $\mathbf{x}_t$ at time $t$ and the hidden state matrix $\mathbf{H}_{t-1}$ and memory cell $ \mathbf{c}_{t-1}$ up to $t-1$, we formulate the iterative update process in a MV-LSTM layer as follows:
\begin{align}
&\mathbf{J}_t = \tanh \left( 
         \boldsymbol{\mathcal{W}}_h \circledast_N \mathbf{H}_{t-1}
         + \boldsymbol{\mathcal{W}}_x * \mathbf{x}_t
         + \mathbf{b}_j \right)  \label{eq:j} \\
&\begin{bmatrix}
        \mathbf{i}_t \\
        \mathbf{f}_t \\
        \mathbf{o}_t 
        \end{bmatrix} =  \sigma \left( \mathbf{W}  [\mathbf{x}_t \oplus \text{vec}(\mathbf{H}_{t-1})] + \mathbf{b} \right) \label{eq:gate}\\
&\mathbf{c}_t = \mathbf{f}_t \odot \mathbf{c}_{t-1} + \mathbf{i}_t \odot \text{vec}(\mathbf{J}_t) \label{eq:c}\\
&\mathbf{H}_t = \text{matricization}( \mathbf{o}_t \odot \tanh(\mathbf{c}_t) ) \label{eq:h}
\end{align}
Overall, Eq. \eqref{eq:j} gives rise to the cell update matrix $\mathbf{J}_t = [\mathbf{j}_t^1 \, , \cdots, \, \mathbf{j}_t^N]^\top \in \mathbb{R}^{N \times d} $, where $ \mathbf{j}_t^n \in \mathbb{R}^{d}$ corresponds to the update w.r.t. input variable $n$. 
The term $\boldsymbol{\mathcal{W}}_h \circledast_N \mathbf{H}_{t-1} $ and $\boldsymbol{\mathcal{W}}_x * \mathbf{x}_t$ respectively capture the update from the hidden states of the previous step and the new input.
Concretely, the tensor-dot operation $\circledast_{ (\cdot) }$ in Eq. \eqref{eq:j} returns the product of two tensors along a specified axis. 
Thus, given tensor $\boldsymbol{\mathcal{W}}_h$ and $\mathbf{H}_{t-1}$, the tensor-dot of $\boldsymbol{\mathcal{W}}_h$ and $\mathbf{H}_{t-1}$ along the axis $N$ is expressed as $\boldsymbol{\mathcal{W}}_h \circledast_N \mathbf{H}_{t-1} = [ \mathbf{W}_h^1 \mathbf{h}_{t-1}^1 \, , \cdots, \, \mathbf{W}_h^N \mathbf{h}_{t-1}^N ]^\top $, where $ \mathbf{W}_h^n \mathbf{h}_{t-1}^n  \in \mathbb{R}^{d}$.
Additionally, we define $*$ as the product between the transition matrix and input vector:
$ \boldsymbol{\mathcal{W}}_x * \mathbf{x}_t = [ \mathbf{W}_x^1 x_{t}^1 \, , \cdots, \,  \mathbf{W}_x^N x_{t}^N]^\top$.



Eq. \eqref{eq:gate} derives the input gate $\mathbf{i}_t$, forget gate $\mathbf{f}_t$ and output gate $\mathbf{o}_t$ by using $\mathbf{x}_t$ and $\mathbf{H}_{t-1}$. All these gates are vectors of dimension $D$.
$\text{vec}(\cdot)$ refers to the vectorization operation, where in Eq. \eqref{eq:gate} it concatenates columns of $\mathbf{H}_{t-1}$ into a vector of dimension $D$. $\oplus$ is the concatenation operation. $\sigma(\cdot)$ represents the element-wise sigmoid activation function.
Each element in gate vectors is derived based on $\mathbf{x}_t \oplus \text{vec}(\mathbf{H}_{t-1})$ that carries information regarding all input variables, so as to utilize the cross-correlation between input variables.


In Eq. \eqref{eq:c}, memory cell vector $\mathbf{c}_t$ is updated by using the previous cell 
$\mathbf{c}_{t-1}$ and vectorized cell update matrix $\mathbf{J}_t$ obtained in Eq. \eqref{eq:j}. $\odot$ denotes element-wise multiplication.
Finally, in Eq. \eqref{eq:h} the new hidden state matrix at $t$ is the matricization \footnote{ In our case, matricization is the operation that reshape a vector of $\mathbb{R}^D$ into a matrix of $\mathbb{R}^{N \times d}$.}
of $\tanh(\mathbf{c_t})$ weighted by the output gate.

\subsection{Mixture Temporal and Variable Attention}
After feeding a sequence of $\{ \mathbf{x}_1, \cdots, \mathbf{x}_T \}$ into MV-LSTM,
we obtain a sequence of hidden state matrices, denoted by $\boldsymbol{\mathcal{H}}_{T} =[\mathbf{H}_1, \cdots, \mathbf{H}_{T}]$, where $\boldsymbol{\mathcal{H}}_T \in \mathbb{R}^{T \times N \times d}$ and element $\mathbf{H}_t \in \mathbb{R}^{N \times d}$. 
$\boldsymbol{\mathcal{H}}_T$ is then used in our mixture temporal and variable attention mechanism, which facilitates the following learning, inference and interpretation of variable importance.  

\begin{figure}[htbp!]
\centering
\includegraphics[width=0.84\textwidth]{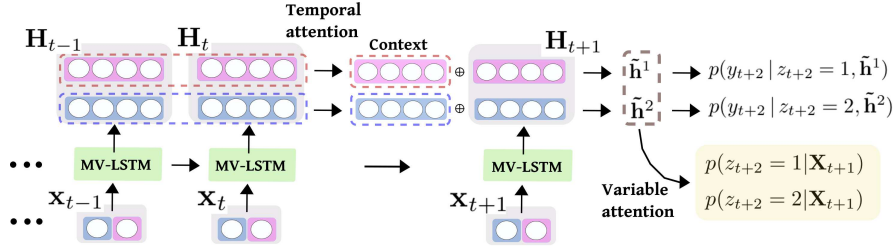}
\caption{Illustration of the mixture temporal and variable attention process in a MV-LSTM layer with a two-variable input sequence and the hidden matrix of size $4 \times 2$, i.e. $4$-dimensional hidden state per variable. (best viewed in color)}
\label{fig:graph}
\end{figure}

Specifically, our attention mechanism is based on a probabilistic mixture of experts model \cite{zong2018deep, graves2013generating, shazeer2017outrageously} over $y_{T+1}$ as:
\begin{align}\label{eq:y}
\begin{split}
&p(y_{T+1} | \mathbf{X}_{T}) = \sum_{n=1}^{N} p(y_{T+1}, z_{T+1}=n | \mathbf{X}_{T})
= \sum_{n=1}^{N} p(y_{T+1} | z_{T+1} = n, \mathbf{X}_{T} ) p(z_{T+1} = n | \mathbf{X}_{T})
\\ &= \underbrace{ \sum_{n=1}^{N-1} p(y_{T+1} | z_{T+1} = n, \mathbf{X}_{T}^n ) p(z_{T+1} = n | \mathbf{X}_{T})}_{\text{Exogenous part}} + \underbrace{p(y_{T+1} | z_{T+1} = N, \mathbf{Y}_{T} )p(z_{T+1} = N | \mathbf{X}_{T})}_{\text{Autoregressive part}} \, ,
\end{split}
\end{align} 
where $\mathbf{X}_T^n = \{ x_1^n, \cdots, x_T^n \}$ and $\mathbf{Y}_T = \{y_1, \cdots, y_T\}$.


In Eq. \eqref{eq:y}, we introduce a latent random variable $z_{T+1}$ into the the density function of $y_{T+1}$ to govern the generation of $y_{T+1}$ conditional on historical data $\mathbf{X}_{T}$.
$z_{T+1}$ is a discrete variable over the set of values $\{ 1, \cdots, N \}$ corresponding to $N$ input variables.
Mathematically, the joint density of $y_{T+1}$ and $z_{T+1}$ is decomposed into  
a component model (i.e. $p(y_{T+1} | z_{T+1} = n, \mathbf{X}_{T}^{n})$) and the prior of $z_{T+1}$ conditioning on $\mathbf{X}_{T}$ (i.e. $ p(z_{T+1} = n | \mathbf{X}_{T}) $).
The component model characterizes the density of $y_{T+1}$ conditioned on historical data of variable $n$, while the prior of $z_{T+1}$ controls to what extent $y_{T+1}$ is generated by variable $n$ as well as enabling to adaptively adjust the contribution of variable $n$ to fit $y_{T+1}$.
For $z_{T+1} = N$, it refers to the autoregressive part. 

Evaluating each part of Eq. \eqref{eq:y} amounts to the temporal and variable attention process using hidden states $\boldsymbol{\mathcal{H}}_T$ in MV-LSTM.
Temporal attention is first applied to the sequence of hidden states for individual variables,
so as to obtain a summarized hidden state for each variable. 
The history of each variable is encoded in such temporally summarized hidden states, which are used to calculate $p(y_{T+1} | z_{T+1} = n, \mathbf{X}_{T}^{n})$ and $p(y_{T+1} | z_{T+1} = N, \mathbf{Y}_{T})$.
Then, since the prior $p(z_{T+1} = n |\mathbf{X}_{T})$ in \eqref{eq:y} is a discrete distribution on $\{1, \cdots, N\}$, it naturally characterizes the attention on the exogenous and autoregressive parts for predicting $y_{T+1}$.

In detail, 
the weights and bias of the temporal attention process are defined as $\boldsymbol{\mathcal{W}}_s = [ \mathbf{W}_s^1, \cdots, \mathbf{W}_s^N ]^\top$ and $\mathbf{b}_s \in \mathbb{R}^{N}$. $\boldsymbol{\mathcal{W}}_s \in \mathbb{R}^{N \times d}$ and the element $\mathbf{W}_s^n \in \mathbb{R}^{d}$ corresponds to $n$-th variable. 
The temporal attention is then derived as:
\begin{align}
& \mathbf{e} = \tanh( \boldsymbol{\mathcal{H}}_{T-1}  \circledast_N \boldsymbol{\mathcal{W}}_s  + \mathbf{b}_s) \label{eq:e}\\
& \mathbf{a} = [\text{softmax}(\mathbf{e}^1), \cdots, \text{softmax}(\mathbf{e}^N)]^{\top} \label{eq:a}\\
& \mathbf{c}_{att} = \mathbf{a} \circledast_N \boldsymbol{\mathcal{H}}_{T-1} \label{eq:cxt}\\
&  \mathbf{\tilde{H}} = \mathbf{H}_{T} \oplus_d \mathbf{c}_{att} \label{eq:atth}
\end{align}
In Eq. \eqref{eq:e}, $\mathbf{e} = [ \mathbf{e}^1, \cdots, \mathbf{e}^N ]^\top \in \mathbb{R}^{N \times (T-1)}$ is derived via the tensor-dot operation, where element $\mathbf{e}^n \in \mathbb{R}^{T-1}$ is the attention score on previous $T-1$ steps of variable $n$ (other methods of deriving attention scores is compatible with MV-LSTM \cite{cinar2017position, qin2017dual} and we use the simple one layer transformation in the present paper).
Then, the attention weights $\mathbf{a} \in \mathbb{R}^{N \times (T-1)}$ is obtained by performing $\text{softmax}(\cdot)$ on each row of $\mathbf{e}$.
$\mathbf{a} \circledast_N \boldsymbol{\mathcal{H}}_{T-1}$ gives rise to the variable-wise context matrix $\mathbf{c}_{att} \in \mathbb{R}^{N \times d}$.
Recall that the hidden state matrix at $T$ is $\mathbf{H}_T \in \mathbb{R}^{N \times d}$.
By concatenating $\mathbf{c}_{att}$ and $\mathbf{H}_T$ along axis $d$ in Eq. \eqref{eq:atth}, we obtain the context enhanced hidden state matrix $\mathbf{\tilde{H}} = [\mathbf{\tilde{h}}^1, \cdots, \mathbf{\tilde{h}}^N ]^\top \in \mathbb{R}^{N \times 2d}$, where $\mathbf{\tilde{h}}^n \in \mathbb{R}^{2d}$ is a hidden state summarizing the temporal information of variable $n$.  

Now we can formulate individual component model in Eq. \eqref{eq:y} as:
\begin{align}
p(y_{T+1} | z_{T+1} = n, \mathbf{X}_{T}^n ) \approx p(y_{T+1} | \mathbf{\tilde{h}}^n ) = \mathcal{N}( y_{T+1} \, | \, \mathbf{W}_o^n \cdot \mathbf{\tilde{h}}^n + b_{o}^n, \sigma^2  )\,,
\end{align}
where we impose normal distribution over $y_{T+1}$, and $\mathbf{W}_o^n$ and $b_{o}^n$ are output weight and bias. In experiments, we simply set $\sigma^2$ to one.   
Meanwhile, by using summarized hidden states $\mathbf{\tilde{H}} $, we derive $p(z_{T+1} = n | \mathbf{X}_{T} )$ to characterize variable level attention as:
\begin{align}p(z_{T+1} = n | \mathbf{X}_{T} ) \approx p(z_{T+1} = n | \mathbf{\tilde{H}} ) = \frac{\exp(\tanh( \mathbf{W}_v^\top \mathbf{\tilde{h}}^n + b_v ))}{\sum_{k=1}^N \exp(\tanh( \mathbf{W}_v^\top \mathbf{\tilde{h}}^k + b_v ))}\, ,\end{align}
where $\mathbf{W}_v \in \mathbb{R}^{2d}$ is the variable attention weight and $b_v \in \mathbb{R}$ is the bias.

\subsection{Learning, Inference and Interpretation}
In the learning phase, denote by $\Theta = \{ \boldsymbol{\mathcal{W}}_h, \boldsymbol{\mathcal{W}}_x, \mathbf{b}, \mathbf{b}_j,  \mathbf{W}, \boldsymbol{\mathcal{W}}_s, \mathbf{W}_v, \mathbf{W}_o, b_o, b_v \}$ the set of parameters in MV-LSTM. Given a set of $M$ training sequences $\{\mathbf{X}_{T}\}_m$ and $\{y_{T+1}\}_m$, the loss function to optimize is defined based on the negative log likelihood of the mixture model plus the regularization term as:
\begin{align}\label{eq:loss}
\mathcal{L}(\Theta) = - \sum_{m=1}^{M} \log \sum_{n=1}^{N} p(z_{\, T+1, m} = n | \mathbf{X}_{T, \, m} ) \mathcal{N}(  y_{T+1 \, , \, m} \, | \, \mathbf{W}_o^n \cdot \mathbf{\tilde{h}}_{m}^n + b_{o}^n \, , \, \sigma^2)  + \lambda \Vert \Theta \Vert^2
\end{align}

In the inference phase, the prediction of $y_{T+1}$ is obtained by the weighted sum of means as \cite{graves2013generating, bishop1994mixture}:
$\hat{y}_{T+1} = \sum_{n=1}^N p(z_{T+1} = n | \mathbf{X}_{T}^n ) ( \mathbf{W}_o^n \cdot \mathbf{\tilde{h}}^n + b_{o}^n )$.

For the interpretation of the variable importance via mixture attention,
we consider to use the posterior of $z_{T+1, m}$, i.e. 
\begin{equation}
p(z_{T+1, m} = n | \mathbf{X}_{T, m}, y_{T+1, m}) \propto p(y_{T+1, m} | z_{T+1, m} = n | \mathbf{X}_{T, m}^n) \cdot p(z_{T+1, m} = n | \mathbf{X}_{T, m} )
\end{equation}
which takes the prediction performance of individual variables into account.
We refer to the derived $p(z_{T+1, m} = n | \mathbf{X}_{T, m}, y_{T+1, m})$ and $p(z_{T+1, m} = n | \mathbf{X}_{T, m} )$ respectively as \textit{posterior} and \textit{prior attention}. 

Meanwhile, note that we obtain the posterior of $z_{T+1, m}$ for each training sequence.  
In order to attain a uniform view of variable importance over the set of data, we define the importance of an input variable by aggregating all the posterior attention of this variable as follows:
\begin{minipage}{.5\linewidth}
\begin{equation*}
\text{Importance}(n) = \frac{\sum_m {p(z_{T+1, m} = n | \mathbf{X}_{T, m} \, , y_{T+1, m})}  }{\sum_k\sum_m {p(z_{T+1, m} = k | \mathbf{X}_{T, m} \, , y_{T+1, m})} },
\end{equation*}
\end{minipage}%
\begin{minipage}{.5\linewidth}
\begin{equation}\label{eq:import}
\sum_{n=1}^{N} \text{Importance}(n) =1
\end{equation}
\end{minipage}


\section{Experiments}
In this part, we report experimental results.
Due to the page limitation, please refer to the appendix section for full results. 
\subsection{Datasets}
 We use three real datasets\footnote{https://archive.ics.uci.edu/ml/datasets.html} and one synthetic dataset to evaluate MV-LSTM and baselines.

\textbf{PM2.5:}
It contains hourly PM2.5 data and the associated meteorological data in Beijing of China. PM2.5 measurement is the target series. The exogenous time series include dew point, temperature, pressure, combined wind direction, cumulated wind speed, hours of snow, and hours of rain. Totally we have $41,700$ multi-variable sequences.

\textbf{Energy:} It collects the appliance energy use in a low energy building. The target series is the energy data logged every 10 minutes. Exogenous time series consist of $14$ variables, e.g. the house inside temperature conditions and outside weather information including temperature, wind speed, humanity and dew point from the nearest weather station.
The number of sequence is $19,700$.

\textbf{Plant:}
This dataset records the time series of energy production of a photovoltaic (PV) power plant in Italy \cite{ceci2017predictive}. Exogenous data consists of $8$ dimensional time series regarding weather conditions (such as temperature, cloud coverage, etc.).
It gives $8,600$ sequences for evaluation. 

\textbf{Synthetic:} It is generated based on the idea of Lorenz model \cite{tank2017interpretable, tank2018neural}.
Exogenous series are generated via the ARMA process with randomized parameters. 
The target series is driven by an ARMA process plus coupled exogenous series of variable $2$ and $3$ with randomized autoregressive orders and thus the synthetic dataset has ground truth of variable importance.  
In total, we generate $40,000$ sequences of 10 exogenous time series.

For each dataset, we perform Augmented Dickey-Fuller (AD-Fuller) and Kwiatkowski Phillips Schmidt Shin (KPSS) tests to determine the necessity of differencing time series \cite{kirchgassner2012introduction}.
The window size, namely $T$ in Sec.~\ref{sec:mv}, is set to 30. 
We further study the prediction performance under different window sizes in the supplementary material.
Each dataset is split into training ($70\%$), validation ($10\%$) and testing sets ($20\%$).
\subsection{Baselines and Evaluation Setup}
The first category of statistics baselines includes: 

\textbf{STRX} is the structural time series model with exogenous variables \cite{scott2014predicting, radinsky2012modeling}. 
It is formulated in terms of unobserved components via the state space method.

\textbf{ARIMAX} augments the classical time series autoregressive integrated moving average model (ARIMA) by adding regression terms on exogenous variables \cite{hyndman2014forecasting}.

The second category of machine learning baselines includes popular tree ensemble methods and regularized regression as:

\textbf{RF} refers to random forests.
It is an ensemble learning method consisting of several decision trees \cite{liaw2002classification, meek2002autoregressive} and has been used in time series prediction \cite{patel2015predicting}.

\textbf{XGT} refers to the extreme gradient boosting \cite{chen2016xgboost}. 
It is the application of boosting methods to regression trees \cite{friedman2001greedy}. 

\textbf{ENET} represents Elastic-Net, which is a regularized regression method combining both L1 and L2 penalties of the lasso and ridge methods \cite{zou2005regularization} and has been used in time series analysis \cite{liu2010learning, bai2008forecasting}.

The third category of neural network baselines includes:

\textbf{RETAIN} requires to pre-train two recurrent neural networks to respectively derive weights on temporal steps and variables, which are then used to perform prediction \cite{choi2016retain}.

\textbf{DUAL} is built upon encoder-decoder architecture \cite{qin2017dual}, which uses an encoder LSTM to learn weights on input variables and then feeds pre-weighted input data into a decoder LSTM for forecasting.

\textbf{cLSTM} proposes to identify Granger causal variables via sparse regularization on the weights of LSTM \cite{tank2017interpretable, tank2018neural}. 

Additionally, we have two variants of MV-LSTM denoted by \textbf{MV-Indep} and \textbf{MV-Fusion}, which are developed to evaluate the efficacy of the updating and mixture mechanism of MV-LSTM.
\textbf{MV-Indep} builds independent recurrent neural networks for each input variable, whose outputs are fed into the mixture attention process to obtain prediction. 
The only difference between \textbf{MV-Fusion} and MV-LSTM is that, 
instead of using mixture attention, MV-Fusion fuses the hidden states of each variable into one hidden state via variable attention.

In ARIMAX, the orders of auto-regression and moving-average terms are set via the autocorrelation and partial autocorrelation.
For RF and XGT, the hyper-parameter tree depth and the number of iterations are chosen from range $[3, 10]$ and $[3, 200]$ via grid search. For XGT, L2 regularization is added by searching within $\{0.0001, 0.001, 0.01, 0.1, 1, 10\}$.
As for ENET, the coefficients for L2 and L1 penalties are selected from $\{0, 0.1, 0.3, 0.5, 0.7, 0.9, 1, 1.5, 2\}$. For these machine learning baselines, multi-variable input sequences are flattened into feature vectors. 

We implemented MV-LSTM and neural network baselines with Tensorflow\footnote{ Code will be released upon requested.}. 
For training, we used Adam with the mini-batch of $64$ instances \cite{kingma2014adam}.
For the size of recurrent and dense layers in the baselines, we conduct grid search over $\{16, 32, 64, 128, 256, 512\}$. 
The size of the MV-LSTM recurrent layer is set by the number of neurons per variable selected from $\{10, 15, 20, 25\}$.
Dropout is set to $0.5$.
Learning rate is searched in $\{0.0005, 0.001, 0.005, 0.01, 0.05 \}$.
L2 regularization is added with the coefficient chosen from $\{0.0001, 0.001, 0.01, 0.1, 1.0\}$.
We train each approach $10$ times and report average performance.

We consider two metrics to measure the prediction performance. Specifically, RMSE is defined as $ \text{RMSE} = \sqrt[]{\sum_{t} (y_t - \hat{y}_t)^2/M} $. MAE is defined as $ \text{MAE} = \sum_{t}|y_t - \hat{y}_t|/M$.
\subsection{Prediction Performance}
We report the prediction errors of all approaches in Table~\ref{tab:rmse} and Table~\ref{tab:mae}. 
In Table~\ref{tab:rmse}, we observe that in most of the time, STRX and ARIMAX underperform machine learning and neural network solutions. Among RF, XGT, and ENET, XGT performs the best mostly. As for neural network baselines, DUAL outperforms RETAIN and cLSTM as well as machine learning baselines in the Synthetic and Energy datasets.
Our MV-LSTM outperforms baselines by around $40\%$ at most.
MV-LSTM performs slightly better than both of MV-Fusion and MV-Indep, while providing the interpretation benefit, which is shown in the next group of experiments.
Above observations also apply to the MAE results in Table~\ref{tab:mae} and we skip the detailed description.
\begin{table}[!h]
  \centering
  \caption{Average test RMSE and std. errors}
  \begin{tabular}{|c|c|c|c|c|}
    \hline
     \backslashbox{Methods}{Dataset} & Synthetic & Energy & Plant & PM2.5  \\
    \hline
    STRX&        $9.23 \pm 0.12$ & $56.87 \pm 0.09$&  $249.89 \pm 0.13$ & $52.51 \pm 0.22$ \\
    ARIMAX&      $9.04 \pm 0.03$ & $50.04 \pm 0.06$&  $223.72 \pm 0.15$ & $42.51 \pm 0.13$ \\
    \hline
    RF &         $6.15 \pm 0.04$ & $49.64 \pm 1.80$&  $165.71 \pm 0.15$ &  $33.84 \pm 1.13$ \\
    XGT &        $6.06 \pm 0.03$ & $41.15 \pm 0.06$& $166.65 \pm 0.09$ & $25.00 \pm 0.02$ \\
    ENET &       $6.05 \pm 0.01$ & $42.78 \pm 0.11$&  $173.26 \pm 0.11$ & $26.03 \pm 0.19$ \\
    \hline
    DUAL &       $6.17 \pm 0.05$ & $40.30 \pm 0.11$ & $175.34 \pm 0.45$  & $25.53 \pm 0.08$ \\
    RETAIN &     $6.18 \pm 0.01$ & $54.77 \pm 0.11$ & $286.64 \pm 0.25$  & $61.22 \pm 0.49$ \\
    cLSTM &      $6.37 \pm 0.01$ & $65.42 \pm 0.23$ & $167.56 \pm 0.34$  & $83.59 \pm 0.05$ \\
    \hline
    MV-Fusion &   $6.02 \pm 0.03$ & $41.26 \pm 0.06$& $162.63 \pm 0.52$ & $25.94 \pm 0.06$ \\
    MV-Indep &    $6.06 \pm 0.05$ & $40.20. \pm 0.07$& $159.90 \pm 0.22$ & $25.15 \pm 0.12$ \\
    \hline
    MV-LSTM &     $\textbf{5.92} \pm 0.03$ & $\textbf{39.81} \pm 0.03$& $\textbf{157.23} \pm 0.16$& $\textbf{24.79} \pm 0.09$ \\
    \hline
\end{tabular}
\label{tab:rmse}
\end{table}
\begin{table}[!h]
  \centering
  \caption{Average test MAE and std. errors}
  \begin{tabular}{|c|c|c|c|c|}
    \hline
     \backslashbox{Methods}{Dataset} & Synthetic & Energy & Plant & PM2.5  \\
    \hline
    STRX&        $7.44 \pm 0.11$ & $68.77 \pm 0.10$ &  $201.03 \pm 0.20$ & $53.93 \pm 0.12$ \\
    ARIMAX&      $7.32 \pm 0.04$ & $50.43 \pm 0.04$ &  $174.09 \pm 0.18$ & $40.05 \pm 0.10$ \\
    \hline
    RF &         $5.23 \pm 0.03$ & $27.30 \pm 0.23$  &  $134.84 \pm 0.18$ & $22.27 \pm 0.03$ \\
    XGT &        $5.15 \pm 0.01$ & $20.20 \pm 0.04$  &  $133.69 \pm 0.10$ & $15.72 \pm 0.04$ \\
    ENET &       $5.17 \pm 0.04$ & $21.68 \pm 0.09$  &  $139.00 \pm 0.09$ & $15.92 \pm 0.02$ \\
    \hline
    DUAL &       $5.10 \pm 0.03$ & $20.75 \pm 0.17$ & $139.16 \pm 0.40$   & $16.02 \pm 0.12$ \\
    RETAIN &     $5.25 \pm 0.01$ & $29.70 \pm 0.12$ & $235.98 \pm 0.27$   & $44.50 \pm 0.32$ \\
    cLSTM &      $5.41 \pm 0.01$ & $35.88 \pm 0.11$ & $142.97 \pm 0.14$  & $65.71 \pm 0.09$ \\
    \hline
    MV-Fusion &   $5.17 \pm 0.07$ & $21.13 \pm 0.17$ & $133.64 \pm 1.30$ & $16.06 \pm 0.08$ \\
    MV-Indep &    $5.25 \pm 0.02$ & $19.92 \pm 0.11$ & $138.23 \pm 0.64$ & $15.55 \pm 0.11$ \\
    \hline
    MV-LSTM &     $\textbf{4.98} \pm 0.01$ & $\textbf{19.79} \pm 0.15$ & $\textbf{132.54} \pm 0.21$ & $\textbf{15.24} \pm 0.04$ \\
    \hline
\end{tabular}
\label{tab:mae}
\end{table}
\subsection{Model Interpretation}
In this part, we compare MV-LSTM to baselines also with interpretability over the variable importance, i.e. DUAL, RETAIN and cLSTM. For real datasets without ground truth about variable importance, we perform Granger causality test \cite{arnold2007temporal} to identify causal variables, which are considered as important variables for the further comparison. 
For the synthetic dataset, we evaluate by observing whether an approach can recognize variable $2$ and $3$ with high importance value.  

Similar to MV-LSTM, we can collect variable attentions of each sequence in DUAL and RETAIN and obtain importance value by Eq. \eqref{eq:import}. Note that variable attentions obtained in RETAIN are unnormalized values. In cLSTM, we identify important variables by non-zero corresponding weights of the neural network \cite{tank2018neural} and thus have no importance value to report in Table \ref{tab:rank}. 

Table \ref{tab:rank} reports some top variables ranked by the corresponding importance value in the brackets. The higher the importance value, the more crucial the variable. 
In dataset PM2.5, three variables (i.e. dew point, cumulated wind speed, and pressure) identified as Granger causal variables are also top ranked by the variable importance in MV-LSTM. 
As is pointed out by \cite{liang2015assessing}, dew point and pressure are
the most influential. Strong wind can bring dry and fresh air and it is crucial as well. This is in line with the variable importance detected by MV-LSTM. On the contrary, baselines miss identifying some variables.
Likewise, for Plant dataset, as is suggested by \cite{mekhilef2012effect, ghazi2014effect}
in addition to cloud cover, humidity, wind speed, and temperature affect the efficiency of PV cells and thus important for power generation.  
\begin{table}[!h]
  \centering
  \caption{Interpretation of variable importance.}
\resizebox{1.0\textwidth}{!}{
  \begin{tabular}{|m{1.2cm}|l|l|}
    \hline
    Dataset & Method & Rank of variables according to importance  \\
    \hline
    \multirow{3}{*}{PM2.5}& MV-LSTM & \cb{Dew point}$(0.38)$, \cb{Cumulated wind speed}$(0.35)$,                                               \cb{Pressure}$(0.073)$ \\
    & DUAL & Temperature (0.29), \cb{Dew Point}(0.26), \cb{Pressure}(0.21)\\
    & RETAIN & \cb{Pressure}(1.14), Cumulated hours of snow (0.04), \cb{Cumulated wind speed}(-0.42)\\
    & cLSTM &  \cb{Dew Point}, \cb{Pressure}, Temperature \\
    \hline
    \\[-1em]
    \multirow{3}{*}{Plant} & MV-LSTM & \cb{Cloud cover}$(0.26)$, \cb{Wind speed}$(0.12)$, \cb{Temperature}$(0.09)$, \cb{Humidity}(0.07) \\
    & DUAL & \cb{Humidity}$(0.29)$, \cb{Cloud cover}$(0.16)$, \cb{Wind speed}$(0.15)$, \cb{Temperature}(0.14) \\
    & RETAIN & Plant temp.$(0.69)$, \cb{Wind speed}$(0.38)$, Dew point$(0.35)$, \cb{Temperature} (0.25)\\
    & cLSTM &  \cb{Dew point}, \cb{Humidity}, Plant temperature, Wind bearing \\
    \hline
    \\[-1em]
    \multirow{3}{*}{Energy} & MV-LSTM & \cb{Living room temp.}$(0.36)$, Office room temp.$(0.17)$, \cb{Parents room temp.}$(0.17)$ \\
    & DUAL & Humidity outside (0.17), Wind speed (0.16), \cb{Living room temp.}(0.10) \\
    & RETAIN & Building outside temp. (0.13), \cb{Parents room temp.}(0.11), Outside temp. (0.11) \\
    & cLSTM &  Humidity outside, Office room temp., \cb{Living room temp.}\\
    \\[-1em]
    \hline
    \\[-1em]
    \multirow{3}{*}{Synthetic} & MV-LSTM & \cb{Variable 3}(0.18), \cb{ Variable 2}(0.18), Variable 8 (0.17), Variable 6 (0.15), Variable 4 (0.13)\\
    & DUAL & Variable 1 (0.12), Variable 0 (0.12), Variable 7 (0.11), \cb{Variable 3} (0.10), Variable 6 (0.10) \\
    & RETAIN &  Variable 10 (1.08), Variable 8(0.09), Variable 9(0.07), Variable 4 (0.06), Variable 6 (0.05) \\
    & cLSTM &  Variable 7, Variable 6, Variable 0, \cb{Variable 3}, Variable 1 \\
    \hline
\end{tabular}
}
\label{tab:rank}
\raggedright *Color box \cb{$\cdot$} represents the variable is important based on Granger causality test or ground truth.  
\end{table}
\begin{figure*}[htbp!]
    \centering
    \begin{subfigure}[t]{0.325\textwidth}
        \centering
        \includegraphics[width=1\textwidth]{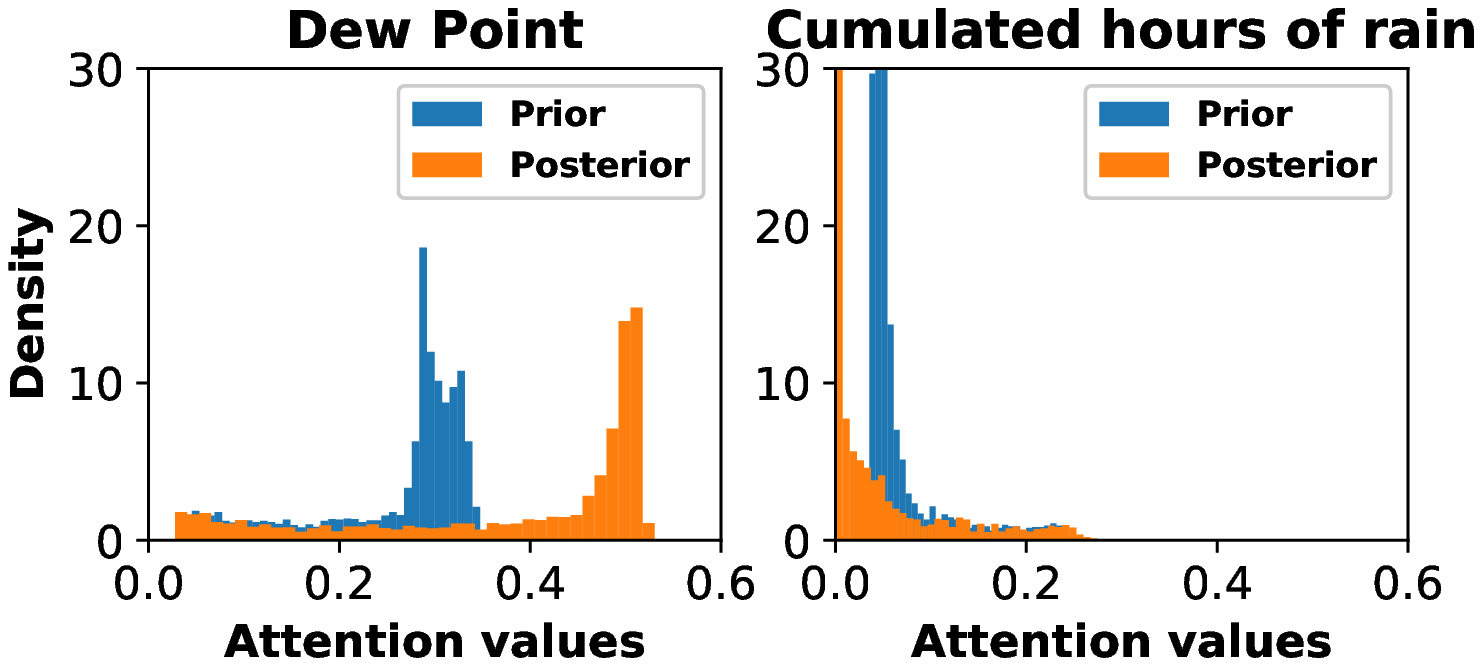}
        \caption{MV-LSTM}
    \end{subfigure}
    \begin{subfigure}[t]{0.325\textwidth}
        \centering
        \includegraphics[width=1\textwidth]{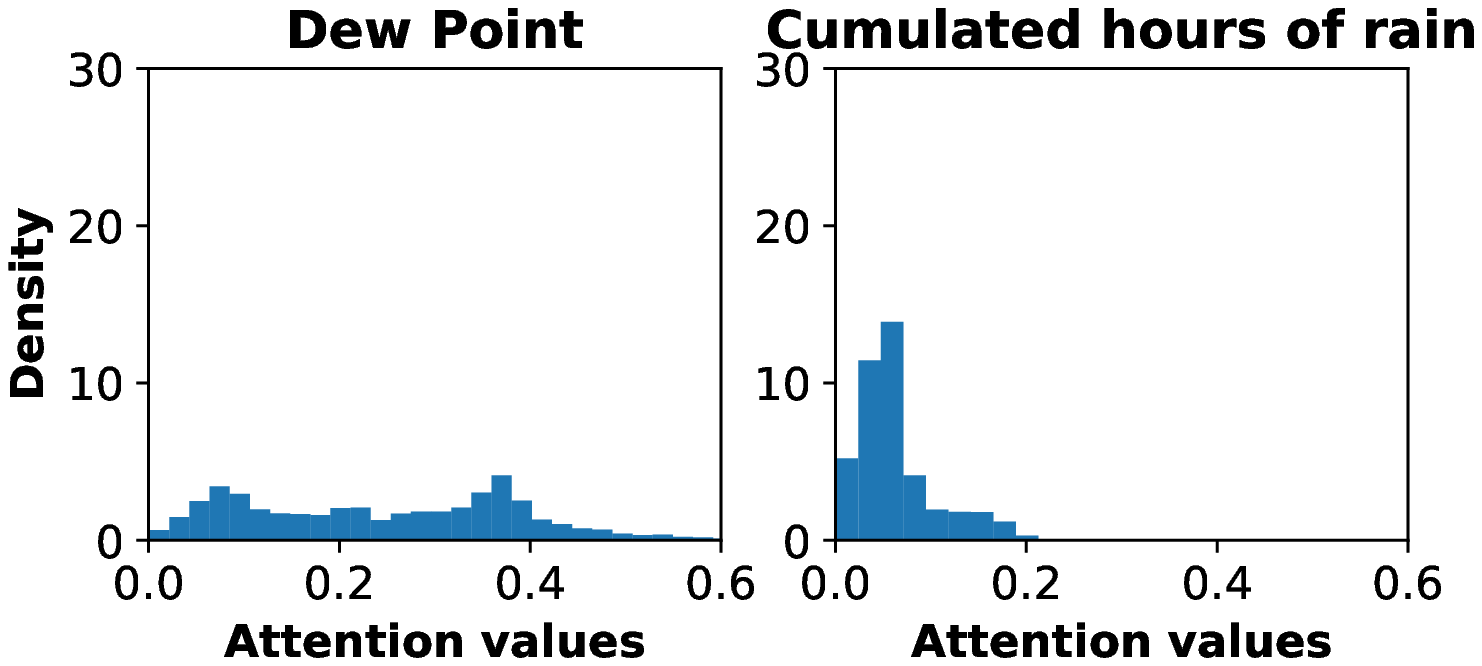}
        \caption{DUAL}
    \end{subfigure}
    \begin{subfigure}[t]{0.325\textwidth}
        \centering
        \includegraphics[width=1\textwidth]{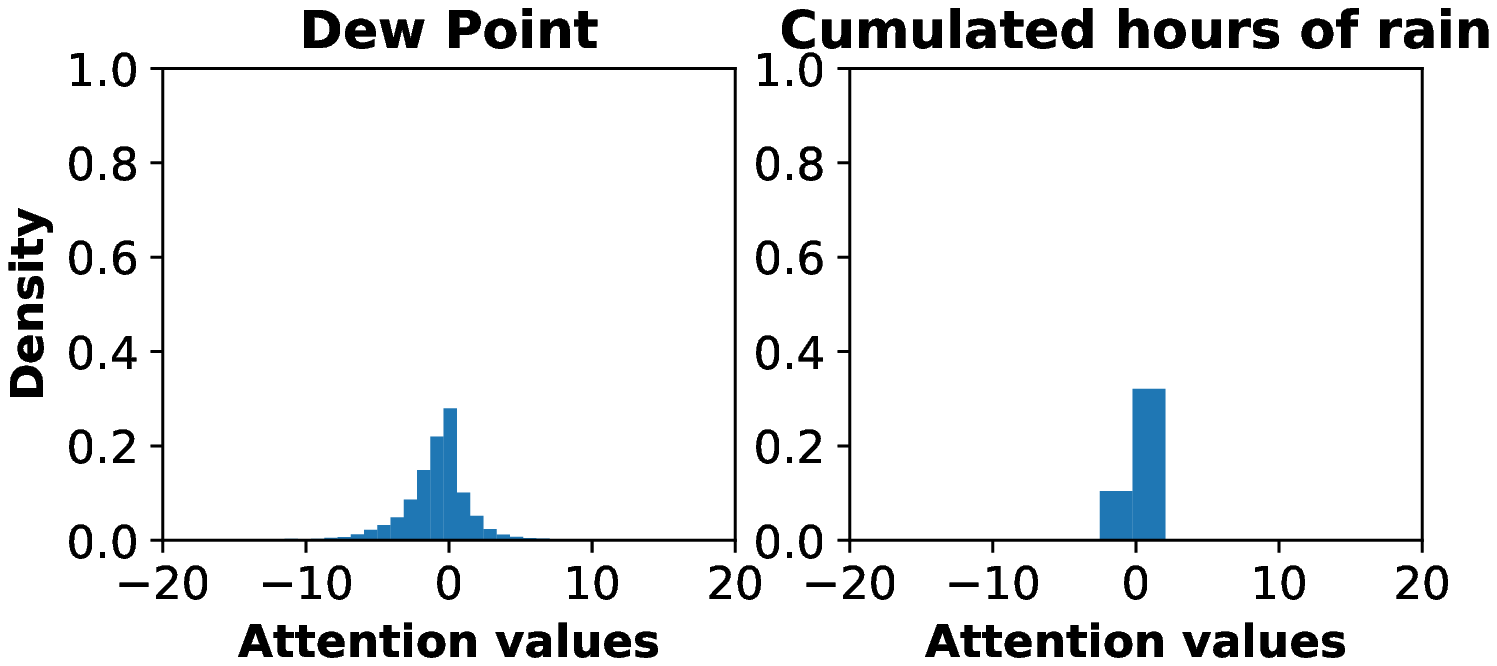}
        \caption{RETAIN}
    \end{subfigure}
    \caption{Histogram visualization of variable attentions w.r.t. two example variables in PM2.5. For MV-LSTM, both prior and posterior attentions are shown. DUAL and RETAIN only have attention weights.}
\label{fig:hist}
\end{figure*}

Furthermore, Figure~\ref{fig:hist} visualizes the histograms of attention values of two example variables in the PM2.5 dataset.
In MV-LSTM, compared with priors, the posterior attention of the variable ``dew point'' shifts rightward, while the posterior of variable ``cumulated hours of rain'' moves towards zero. 
It indicates that posterior attention rectifies the prior by taking into account the predictive likelihood. As a result, the variable importance derived from posterior attention is more distinguishable and informative, compared with the attention weights in DUAL and RETAIN.

\section{Conclusion}
In this paper, we propose an interpretable multi-variable LSTM for time series with exogenous variables. Based on the tensorized hidden states in MV-LSTM, we develop mixture temporal and variable attention mechanism, which enables to infer and quantify the variable importance w.r.t. the target series. 
Extensive experiments on a synthetic dataset with ground truth and real datasets with Granger causality test 
exhibit the superior prediction performance and interpretability of MV-LSTM. 

\bibliography{iclr2018_workshop}
\bibliographystyle{iclr2018_workshop}

\newpage
\section{Appendix}

\subsection{Multi-Variable LSTM}

\begin{theorem}
The hidden states and memory cells in MV-LSTM are updated by the process below:
\begin{align*}
&\mathbf{J}_t = \tanh \left( 
         \boldsymbol{\mathcal{W}}_h \circledast_N \mathbf{H}_{t-1}
         + \boldsymbol{\mathcal{W}}_x * \mathbf{x}_t
         + \mathbf{b}_j \right)    \\
&\begin{bmatrix}
        \mathbf{i}_t \\
        \mathbf{f}_t \\
        \mathbf{o}_t 
        \end{bmatrix} =  \sigma \left( \mathbf{W}  [\mathbf{x}_t \oplus \text{vec}(\mathbf{H}_{t-1})] + \mathbf{b} \right) \\
&\mathbf{c}_t = \mathbf{f}_t \odot \mathbf{c}_{t-1} + \mathbf{i}_t \odot \text{vec}(\mathbf{J}_t) \\
&\mathbf{H}_t = \text{matricization}( \mathbf{o}_t \odot \tanh(\mathbf{c}_t) ) 
\end{align*}
and therefore
each element $\mathbf{h}_t^n$ of hidden state matrix $\mathbf{H}_t = [ \mathbf{h}_t^{1} \, , \cdots, \, \mathbf{h}_t^{N} ]^\top$ encodes the information exclusively from the corresponding input variable $n$.
\end{theorem}
\begin{proof}
By the tensor-dot operation $\circledast_{ N }$, only the elements of $\boldsymbol{\mathcal{W}}_h$ and $\mathbf{H}_{t-1}$ corresponding to $n$ are matched to perform calculation, namely
$\boldsymbol{\mathcal{W}}_h \circledast_N \mathbf{H}_{t-1} = [ \mathbf{W}_h^1 \mathbf{h}_{t-1}^1 \, , \cdots, \, \mathbf{W}_h^N \mathbf{h}_{t-1}^N ]^\top $. 
Meanwhile, since the product between input-hidden transition weights and input vector is
$ \boldsymbol{\mathcal{W}}_x * \mathbf{x}_t = [ \mathbf{W}_x^1 x_{t}^1 \, , \cdots, \,  \mathbf{W}_x^N x_{t}^N]^\top$, each resulting element $\mathbf{W}_x^1 x_{t}^n$ only carries information about variable $n$. 
Then, though the derivation of gates $\mathbf{i}_t$, $\mathbf{f}_t$, and $\mathbf{o}_t$ mix information from all input variables in order to capture cross-correlation among variables, memory cells are updated by multiplication operation between gates and $\mathbf{J}_t$ and therefore the information encoded in $\mathbf{J}_t$ are still specific to each input variable. Likewise, hidden state matrix $\mathbf{H}_t$ derived from the updated memory retain the variable-wise hidden states. 
\end{proof}

\subsection{Prediction Performance}
In addition to the results under window size $30$ in Table \ref{tab:rmse} and \ref{tab:mae}, we report the prediction errors under different window sizes i.e. $T$ in Eq. \eqref{eq:loss}.

\begin{table}[!h]
  \centering
  \caption{Average test RMSE and std. errors under window size $10$}
  \begin{tabular}{|c|c|c|c|c|}
    \hline
     \backslashbox{Methods}{Dataset} & Synthetic & Energy & Plant & PM2.5  \\
    \hline
    STRX&        $4.32 \pm 0.02$ & $51.42 \pm 0.05$&  $214.44 \pm 0.32$ & $ 45.32 \pm 0.13$ \\
    ARIMAX&      $4.23 \pm 0.03$ & $47.04 \pm 0.04$&  $204.34 \pm 0.53$ & $ 45.55\pm 0.09$ \\
    \hline
    RF &         $2.22 \pm 0.06$  & $45.41 \pm 0.11$  &  $163.02 \pm 0.23$ & $29.64 \pm 0.13$ \\
    XGT &        $2.19 \pm 0.06$  & $40.53 \pm 0.09$  &  $163.24 \pm 0.43$  & $25.14 \pm 0.09$ \\
    ENET &       $2.19 \pm 0.03$  & $42.77 \pm 0.05$  &  $164.0  \pm 0.41$ & $25.92 \pm 0.11$ \\
    \hline
    DUAL &       $2.21 \pm 0.01$  & $39.97 \pm 0.04$   & $158.40 \pm 0.23$ & $ 25.60 \pm 0.09$ \\
    RETAIN &     $2.25 \pm 0.02$  & $52.68 \pm 0.05$   & $256.09 \pm 0.12$ & $ 52.43 \pm 0.10$ \\
    cLSTM &      $2.34 \pm 0.04$  & $46.32. \pm 009.$  & $168.32 \pm 0.16$ & $ 31.42 \pm 0.11$ \\
    \hline
    MV-Fusion &  $2.14 \pm 0.03$ & $40.13 \pm 0.11$    & $159.54 \pm 0.34$  & $26.10 \pm 0.09$ \\
    MV-Indep &   $2.21 \pm 0.05$ & $40.43 \pm 0.08$    & $161.34 \pm 0.45$  & $25.85 \pm 0.10$ \\
    \hline
    MV-LSTM &    $2.11 \pm 0.04$ & $39.12 \pm 0.02$    & $154.39 \pm 0.12$  & $24.73. \pm 0.06$ \\
    \hline
\end{tabular}
\label{tab:rmse-10}
\end{table}

\begin{table}[!h]
  \centering
  \caption{Average test MAE and std. errors under window size $10$}
  \begin{tabular}{|c|c|c|c|c|}
    \hline
     \backslashbox{Methods}{Dataset} & Synthetic & Energy & Plant & PM2.5  \\
    \hline
    STRX&        $3.54 \pm 0.04$ & $ 41.32 \pm 0.04$   & $ 189.12 \pm 0.56$     & $ 40.93 \pm 0.10$ \\
    ARIMAX&      $3.76 \pm 0.06$ & $ 39.03 \pm 0.03$   & $ 190.87 \pm 0.81$     & $ 39.84 \pm 0.11$ \\
    \hline
    RF &         $2.13\pm 0.07$  & $26.14 \pm 0.10$    & $133.69 \pm 0.43$   & $21.43 \pm 0.23$ \\
    XGT &        $2.03 \pm 0.08$ & $19.49 \pm 0.03$    & $131.48 \pm 0.80$   & $16.12 \pm 0.15$ \\
    ENET &       $1.88 \pm 0.00$ & $22.04 \pm 0.09$    & $137.16 \pm 0.60$   & $15.94 \pm 0.09$ \\
    \hline
    DUAL &       $1.83 \pm 0.05$  & $19.83 \pm 0.03$   & $ 130.56 \pm 0.30$   & $ 16.37 \pm 0.10$ \\
    RETAIN &     $1.92 \pm 0.04$  & $31.21 \pm 0.10$   & $ 201.34 \pm 0.41$   & $ 36.32 \pm 0.09$ \\
    cLSTM &      $2.01 \pm 0.01$  & $21.90 \pm 0.05$   & $ 134.60 \pm 0.31$   & $ 18.43\pm 0.08$ \\
    \hline
    MV-Fusion &   $2.10 \pm 0.03$ & $20.01 \pm 0.01$   & $131.10 \pm 0.31$   & $16.90 \pm 0.10$ \\
    MV-Indep &    $1.98 \pm 0.01$ & $21.03 \pm 0.03$   & $129.90 \pm 0.12$   & $16.10 \pm 0.08$ \\
    \hline
    MV-LSTM &     $1.83. \pm 0.02$ & $18.89 \pm 0.03$  & $128.21 \pm 0.13$   & $15.40. \pm 0.04$ \\
    \hline
\end{tabular}
\label{tab:mae-10}
\end{table}

\begin{table}[!h]
  \centering
  \caption{Average test RMSE and std. errors under window size $20$}
  \begin{tabular}{|c|c|c|c|c|}
    \hline
     \backslashbox{Methods}{Dataset} & Synthetic & Energy & Plant & PM2.5  \\
    \hline
    STRX&        $4.78 \pm 0.03$   & $55.43 \pm 0.09$  &  $ 231.43 \pm 0.19$ & $ 48.12 \pm 0.05$ \\
    ARIMAX&      $4.58 \pm 0.07$   & $53.32 \pm 0.08$  &  $ 225.54 \pm 0.23$ & $ 43.32 \pm 0.07$ \\
    \hline
    RF &         $4.12 \pm 0.05$   & $48.32 \pm 0.11$  &  $164.23 \pm 0.65$  & $ 30.43 \pm 0.04$ \\
    XGT &        $3.90 \pm 0.02$   & $42.43 \pm 0.09$  &  $164.10 \pm 0.54$  & $ 26.32 \pm 0.05$ \\
    ENET &       $3.91 \pm 0.01$   & $42.12 \pm 0.10$  &  $168.22 \pm 0.49$  & $ 27.12 \pm 0.2$ \\
    \hline
    DUAL &       $3.36 \pm 0.02$   & $39.35 \pm 0.09$  & $175.43 \pm 0.54$   & $ 24.89 \pm 0.05$ \\
    RETAIN &     $3.42 \pm 0.05$   & $53.80 \pm 0.10$  & $280.81 \pm 0.36$   & $ 58.44 \pm 0.03$ \\
    cLSTM &      $3.98 \pm 0.07$   & $40.02 \pm 0.11$  & $174.30 \pm 0.34$   & $ 30.66 \pm 0.05$ \\
    \hline
    MV-Fusion &  $3.43 \pm 0.05$   & $41.54 \pm 0.05$  & $163.03 \pm 0.20$   & $26.10 \pm 0.04$ \\
    MV-Indep &   $3.41 \pm 0.02$   & $40.34 \pm 0.08$  & $162.10 \pm 0.32$   & $26.13 \pm 0.06$ \\
    \hline
    MV-LSTM &    $3.27 \pm 0.03$   & $39.13 \pm 0.0$   & $159.97 \pm 0.12$   & $24.70 \pm 0.08$ \\
    \hline
\end{tabular}
\label{tab:rmse-20}
\end{table}

\begin{table}[!h]
  \centering
  \caption{Average test MAE and std. errors under window size $20$}
  \begin{tabular}{|c|c|c|c|c|}
    \hline
     \backslashbox{Methods}{Dataset} & Synthetic & Energy & Plant & PM2.5  \\
    \hline
    STRX&        $4.10 \pm 0.03$  & $47.43 \pm 0.08$  &  $192.23 \pm 0.43$ & $42.13 \pm 0.03$ \\
    ARIMAX&      $4.02 \pm 0.04$  & $46.30 \pm 0.09$  &  $193.42 \pm 0.41$ & $42.98 \pm 0.07$ \\
    \hline
    RF &         $2.90 \pm 0.03$  & $26.34 \pm 005$   &  $130.90 \pm 0.15$ & $21.05 \pm 0.11$ \\
    XGT &        $2.87 \pm 0.04$  & $19.26 \pm 0.08$  &  $131.47 \pm 0.21$ & $15.71 \pm 0.09$ \\
    ENET &       $2.88 \pm 0.08$  & $27.71 \pm 0.11$  &  $137.04 \pm 0.38$ & $15.92 \pm 0.10$ \\
    \hline
    DUAL &       $2.87 \pm 0.08$  & $19.95 \pm 0.10$  &  $136.87\pm 0.12$   & $15.27 \pm 0.09$ \\
    RETAIN &     $2.90 \pm 0.05$  & $30.12 \pm 0.09$  &  $228.45 \pm 0.15$  & $41.21 \pm 0.06$ \\
    cLSTM &      $2.98 \pm 0.07$  & $23.43 \pm 0.11$  &  $141.14 \pm 0.31$  & $16.87 \pm 0.10$ \\
    \hline
    MV-Fusion &  $3.01 \pm 0.03$  & $20.80 \pm 0.10$  & $132.12 \pm 0.011$  & $16.12 \pm 0.04$ \\
    MV-Indep &   $2.98 \pm 0.03$  & $19.91 \pm 0.09$  & $131.22 \pm 0.032$  & $15.90 \pm 0.07$ \\
    \hline
    MV-LSTM &    $2.84 \pm 0.02$  & $18.91 \pm 0.0$   & $129.90 \pm 0.14$   & $14.89 \pm 0.08$ \\
    \hline
\end{tabular}
\label{tab:mae-20}
\end{table}

\subsection{Model Interpretation}
In this part, we provide the variable list about each dataset in Table \ref{tab:fullvar} and report the full variable importance in Table \ref{tab:fullrank}. 
Figure \ref{fig:full_pm} to \ref{fig:full_syn} visualize the histograms of attention of all variables in each dataset.
In MV-LSTM, compared with priors, the posterior attention rectifies the prior by taking into account the predictive likelihood. 
while the attention weights in DUAL and RETAIN are not representative enough.

\begin{table}[!h]
  \centering
  \caption{Datasets.}
  \begin{tabular}{|c|c|}
    \hline
    Dataset & Variables   \\
    \hline
    PM2.5 &  \makecell{Dew Point, Temperature, Pressure, Cumulated wind speed, \\
                       Cumulated hours of snow, Cumulated hours of rain, PM2.5 measurement} \\
    \hline
    \\[-1em]
    Plant & \makecell{Plant temperature, Cloud cover, Dew point, Humidity, \\ 
                      Temperature, Wind bearing, Wind speed, Power production }\\
    \hline
    \\[-1em]
    Energy & \makecell{Kitchen temperature, Living room temperature, Laundry room temperature,\\
    Office room temperature, Bathroom temperature, Building outside temperature, \\
    Ironing room temperature, Teenager room temperature, Parents room temperature, \\
    Outside temperature, Wind speed, Humidity outside, Dew point, Energy consumption} \\
    \\[-1em]
    \hline
    \\[-1em]
    Synthetic &  Variable 0 to 9, target variable 10\\
    \hline
\end{tabular}
\label{tab:fullvar}
\end{table}

\begin{table}[tbhp!]
  \centering
  \caption{Interpretation of variable importance (full results).}
\resizebox{1.0\textwidth}{!}{
  \begin{tabular}{|m{1.2cm}|l|p{10cm}|}
    \hline
    Dataset & Method & Rank of variables according to importance  \\
    \hline
    \multirow{3}{*}{PM2.5}& MV-LSTM & \cb{Dew point}$(0.38)$, \cb{Cumulated wind speed}$(0.35)$, \cb{Pressure}$(0.073)$, 
    Temperature(0.067), Autoregressive(0.05), Cumulated hours of snow(0.04), Cumulated hours of rain(0.04) \\
    & DUAL & Temperature(0.29), \cb{Dew Point}(0.26), \cb{Pressure}(0.21), 
             \cb{Cumulated wind speed}(0.09), Cumulated hours of snow (0.08), Cumulated hours of rain (0.07) \\
    & RETAIN & {Autoregressive}(3.79), \cb{Pressure}(1.14), Cumulated hours of snow(0.04), \newline \cb{Cumulated wind speed}(-0.42), Cumulated hours of rain (-0.47), \cb{Dew Point} (-1.24), Temperature (-1.80) \\
    & cLSTM &  \cb{Dew Point}, \cb{Pressure}, Temperature \\
    \hline
    \\[-1em]
    \multirow{3}{*}{Plant} & MV-LSTM & \cb{Cloud cover}$(0.26)$, \cb{Wind speed}$(0.12)$, \cb{Temperature}$(0.09)$, \cb{Humidity} (0.07), Autoregressive (0.30), Dew point (0.06), Wind bearing (0.05), Plant-temp.(0.05) \\
    & DUAL & \cb{Humidity}$(0.29)$, \cb{Cloud cover}$(0.16)$, \cb{Wind speed}$(0.15)$, \cb{Temperature} (0.14), Wind bearing (0.09), Plant-temp (0.09), Dew point (0.08) \\
    & RETAIN & Plant temp.$(0.69)$, \cb{Wind speed}$(0.38)$, Dew point$(0.35)$, Autoregressive (0.77), \cb{Temperature} (0.25), Wind bearing (-0.13), \cb{Cloud cover} (-0.46), \cb{Humidity} (-0.84) \\
    & cLSTM & Dew point, \cb{Humidity}, Plant temperature, Autoregressive, Wind bearing, \cb{Wind speed} \\
    \hline
    \\[-1em]
    \multirow{3}{*}{Energy} & MV-LSTM & \cb{Living room}$(0.36)$, Office room$(0.17)$, \cb{Parents room}$(0.17)$, Humidity outside (0.15), Dew point(0.06), Wind speed(0.05), Kitchen temp. (0.02), Bathroom temp (0.02), Teenager room temp. (0.02), Building outside temp.(0.01), Outside temp.(0.01), Autoregressive (0.01), Ironing room temp.(0.008), Laundry room temp.(0.002) \\
    & DUAL & Humidity outside(0.17), Wind speed(0.16), \cb{Living room temp.}(0.10), \cb{Parents room temp.} (0.07), Laundry room temp. (0.06), Bathroom temp. (0.06), Building outside temp. (0.06), Teenager room temp.(0.06), Outside temp.(0.06), Ironing room temp. (0.05), Kitchen temp.(0.05), Dew point (0.05), 'Office room temp.', 0.05)  \\
    & RETAIN & Building outside temp.(0.13), Autoregressive (0.12),  \cb{Parents room temp.}(0.11), Outside temp.(0.11), Ironing room temp.(0.09), Bathroom temp. (0.09), Laundry room temp. (0.09), \cb{Living room temp.}(0.08), Office room temp.(0.06), Dew point(0.06), Teenager room temp. (0.05), Kitchen temp.(0.047), Wind speed (0.03), Humidity outside (-0.07)\\
    & cLSTM & Humidity outside, Office room temp., \cb{Living room temp.}, Laundry room temp., \cb{Parents room temp.}, Dew point, Ironing room temp \\
    \\[-1em]
    \hline
    \\[-1em]
    \multirow{3}{*}{Synthetic} & MV-LSTM & \cb{Variable 3}(0.18), \cb{ Variable 2}(0.18), Variable  8(0.17), Variable 6(0.15), Variable 4 (0.13), Variable 1 (0.08), Variable 7 (0.06), Variable 0 (0.02), Variable 9 (0.01), Autoregressive (0.01), Variable 5 (0.01) \\
    & DUAL & Variable 1 (0.12), Variable 0 (0.12), Variable 7 (0.11), \cb{Variable 3} (0.10), Variable 6 (0.10), \cb{Variable 2} (0.09), Variable 4 (0.09), Variable 8 (0.09), Variable 5 (0.09), Variable 9 (0.09)\\
    & RETAIN &  Variable 10 (1.08), Variable 8(0.09), Variable 9(0.07), Variable 4 (0.06), Variable 6 (0.05), \cb{Variable 2} (0.02), Variable 1 (0.01), Variable 7 (-0.04), Variable 0 (-0.1), Variable 5 (-0.11), \cb{Variable 3} (-0.13)  \\
    & cLSTM &  Variable 7, Variable 6, Variable 0, \cb{Variable 3}, Variable 1 \\
    \hline
\end{tabular}
}
\label{tab:fullrank}
\raggedright *Color box \cb{$\cdot$} represents the variable is important based on Granger causality test or ground truth.
\end{table}

\begin{figure}[htbp!]
\centering
\includegraphics[width=1.0\textwidth]{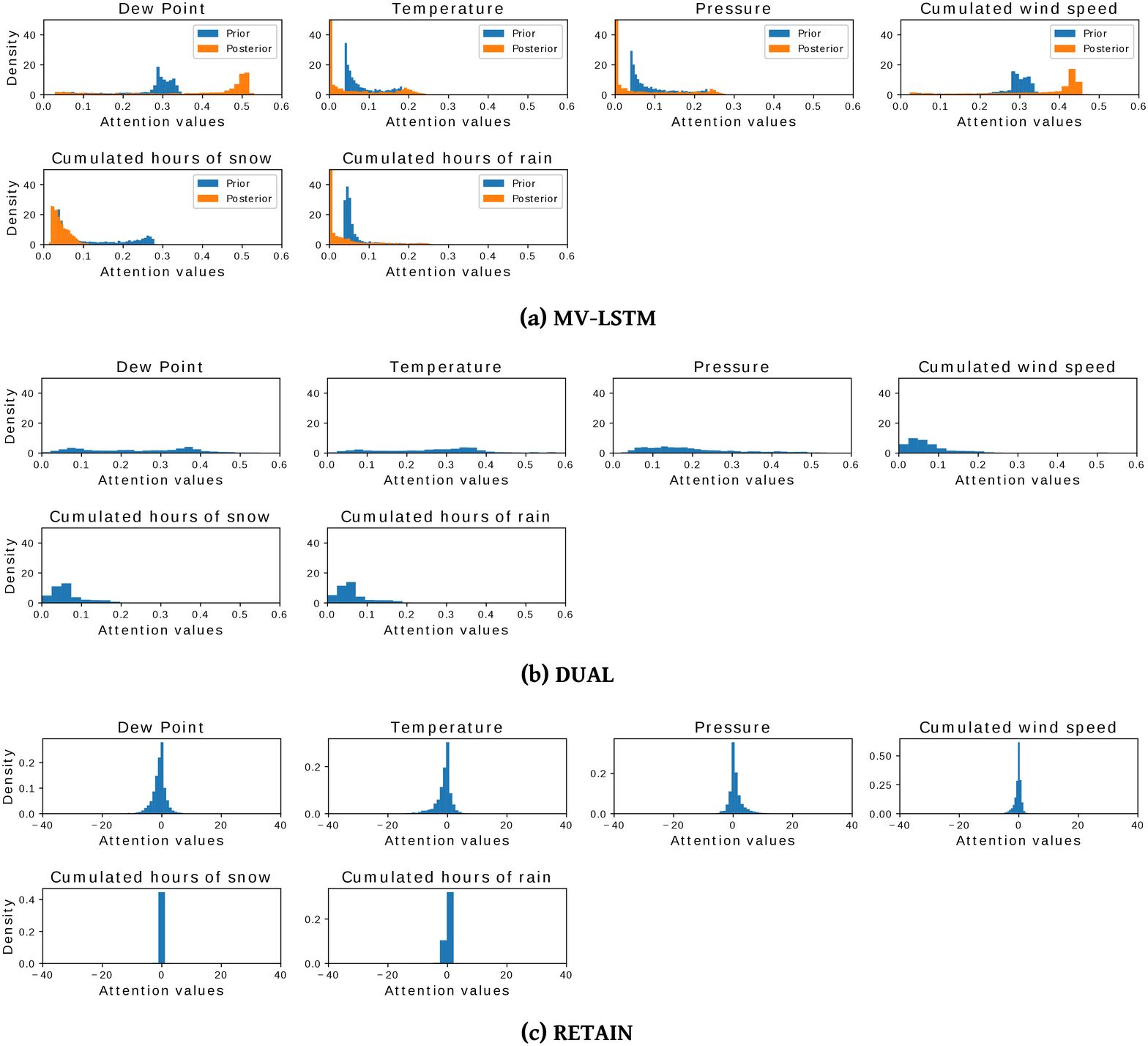}
\caption{Histogram visualization of variable attentions in the PM2.5 dataset. For MV-LSTM, both prior and posterior attentions are shown. DUAL and RETAIN only have attention weights.}
\label{fig:full_pm}
\end{figure}

\begin{figure}[htbp!]
\centering
\includegraphics[width=1.0\textwidth]{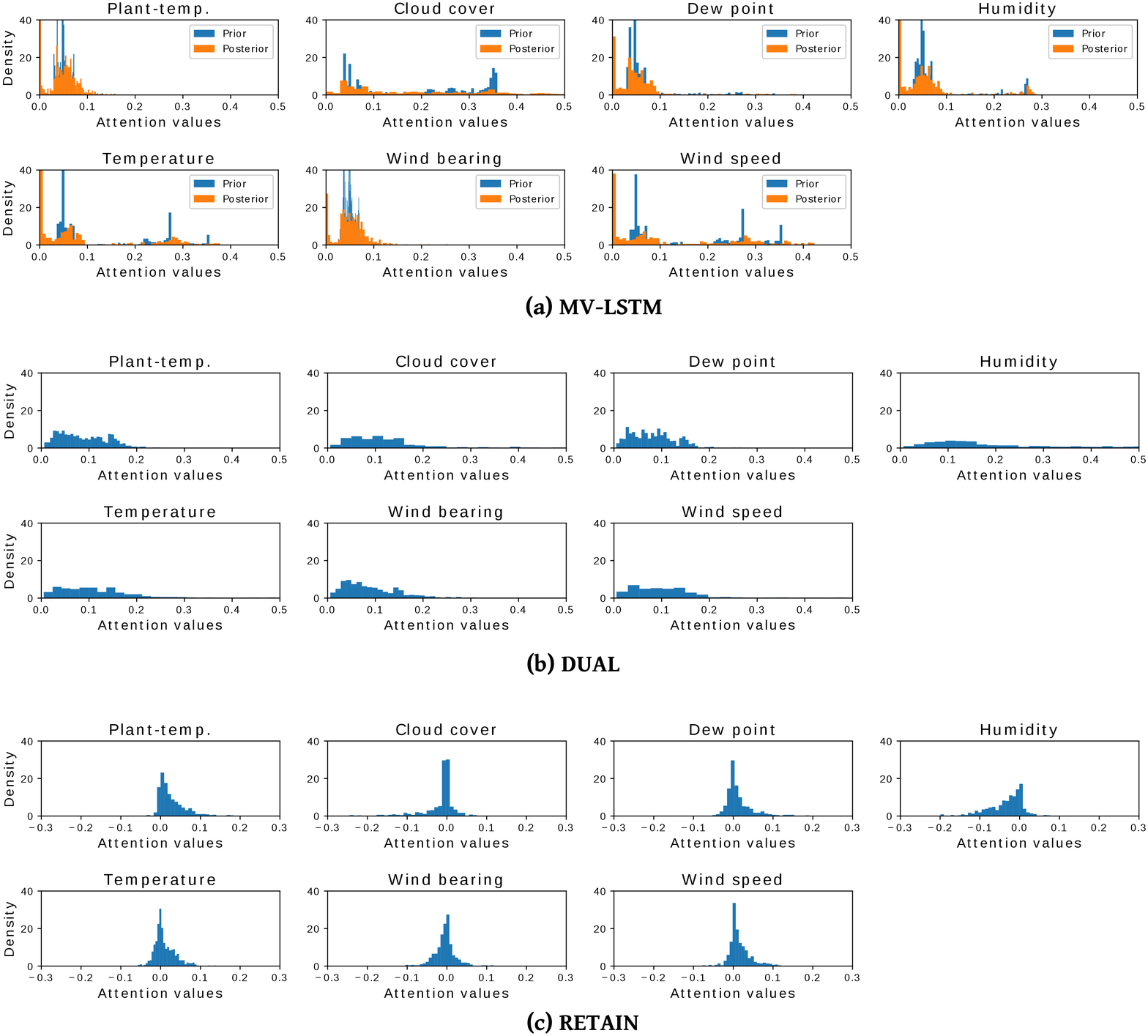}
\caption{Histogram visualization of variable attentions in the Plant dataset. For MV-LSTM, both prior and posterior attentions are shown. DUAL and RETAIN only have attention weights.}
\label{fig:full_plant}
\end{figure}

\begin{figure}[htbp!]
\centering
\includegraphics[width=1.0\textwidth]{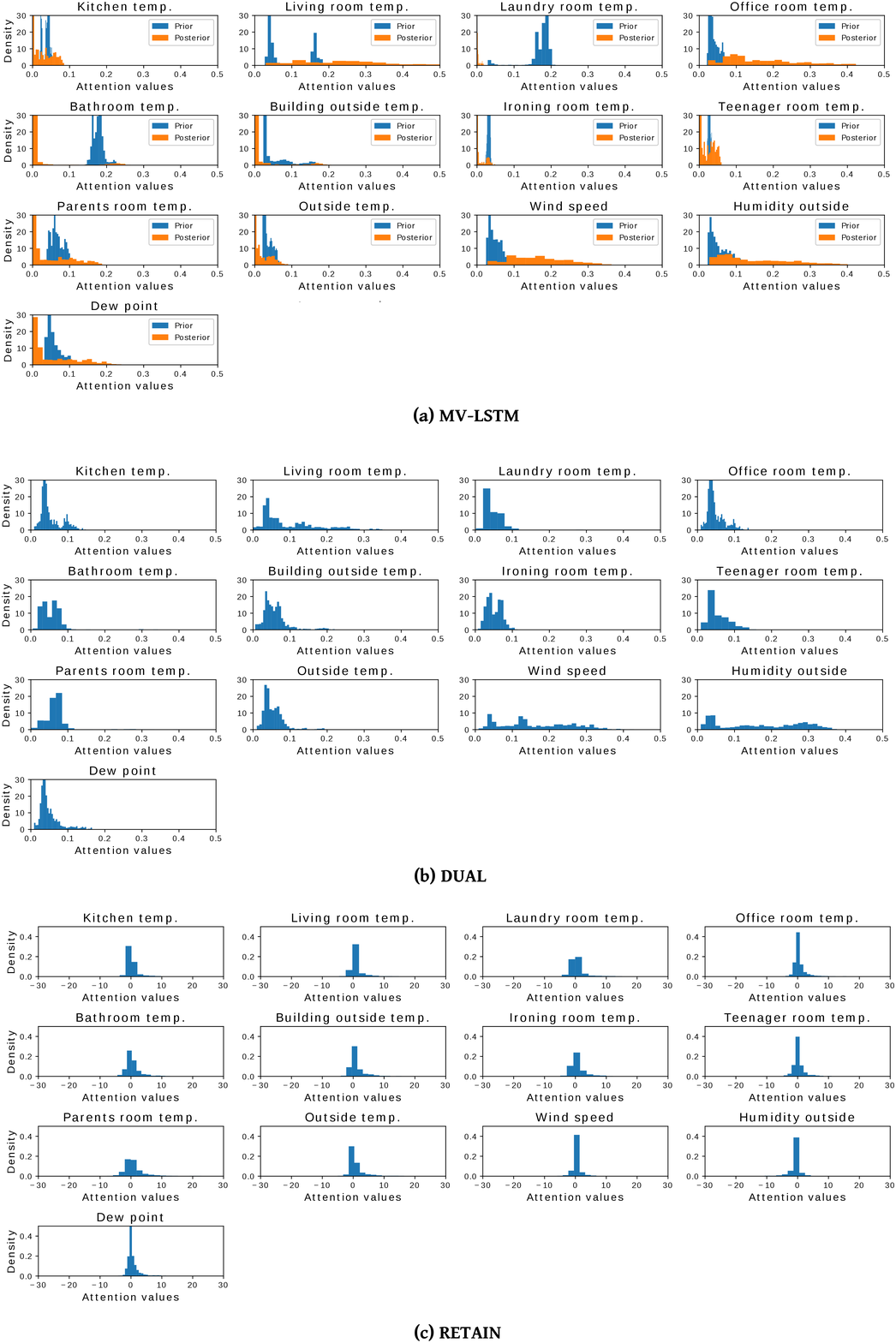}
\caption{Histogram visualization of variable attentions in the Energy dataset. For MV-LSTM, both prior and posterior attentions are shown. DUAL and RETAIN only have attention weights.}
\label{fig:full_energy}
\end{figure}

\begin{figure}[htbp!]
\centering
\includegraphics[width=1.0\textwidth]{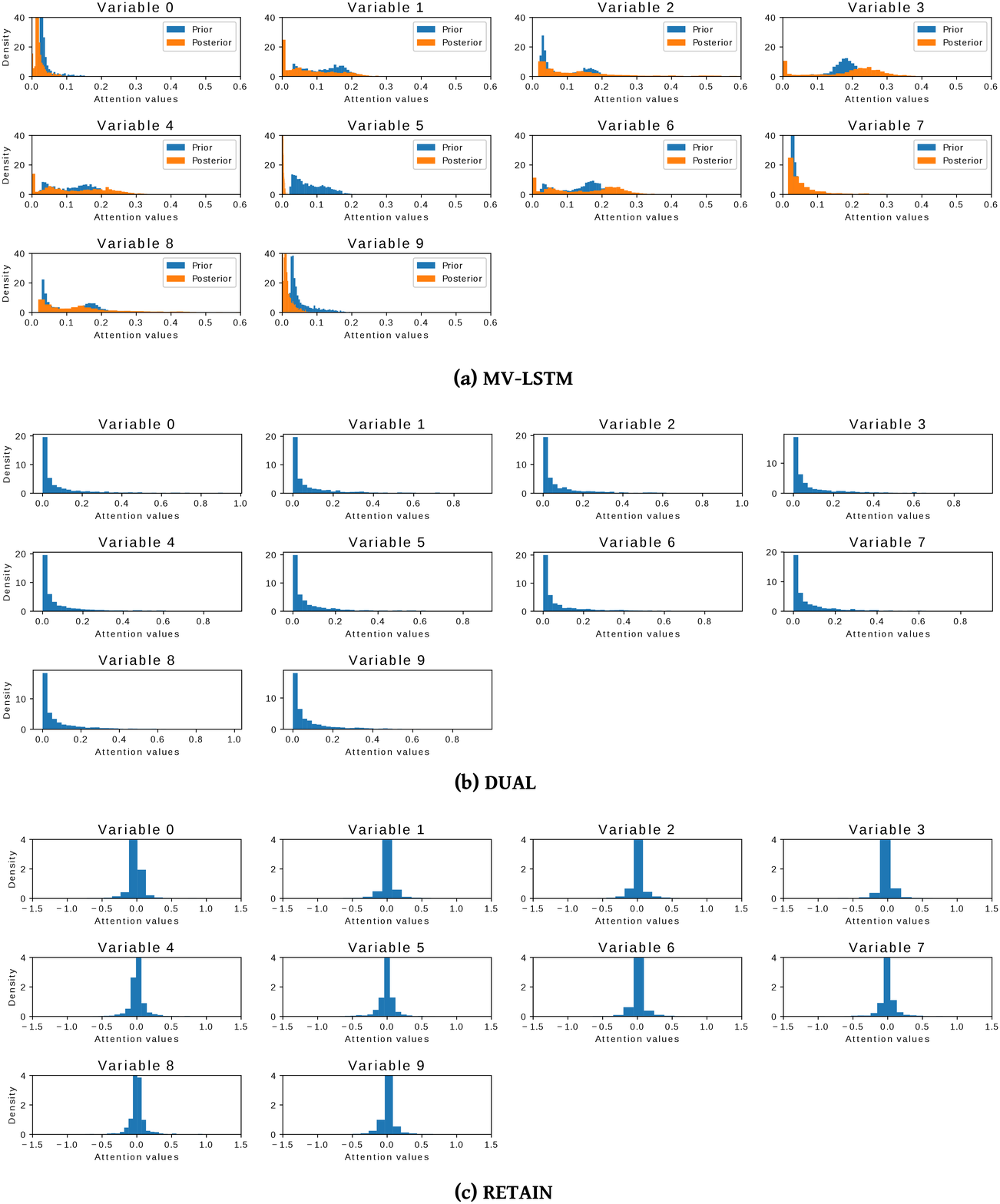}
\caption{Histogram visualization of variable attentions in the Synthetic dataset. For MV-LSTM, both prior and posterior attentions are shown. DUAL and RETAIN only have attention weights.}
\label{fig:full_syn}
\end{figure}

\end{document}